\documentclass{article}

\usepackage{microtype}
\usepackage{graphicx}
\usepackage{booktabs} 
\usepackage{subcaption}

\usepackage{hyperref}

\usepackage[accepted]{icml2023}

\usepackage{amsmath}
\usepackage{amssymb}
\usepackage{mathtools}
\usepackage{amsthm}
\usepackage{amsfonts}
\usepackage{comment}

\usepackage[capitalize,noabbrev,nameinlink]{cleveref}

\usepackage{thmtools}
\usepackage{thm-restate}

\usepackage[textsize=tiny]{todonotes}

\usepackage[utf8]{inputenc}
\usepackage{aliases}
\usepackage{url}
\usepackage{comment}
\usepackage{pifont}
\usepackage{bbold}
\usepackage{dsfont}
\usepackage{thm-restate}
\usepackage[ruled,vlined]{algorithm2e}
\crefname{algocf}{Algorithm}{Algorithms}
\Crefname{algocf}{Algorithm}{Algorithms}
\usepackage[normalem]{ulem}

\usepackage{physics}
\usepackage{amsmath}
\usepackage{tikz}
\usepackage{mathdots}
\usepackage{yhmath}
\usepackage{cancel}
\usepackage{color}
\usepackage{siunitx}
\usepackage{array}
\usepackage{multirow}
\usepackage{amssymb}
\usepackage{gensymb}
\usepackage{tabularx}
\usepackage{extarrows}
\usepackage{booktabs}
\usetikzlibrary{fadings}
\usetikzlibrary{patterns}
\usetikzlibrary{shadows.blur}
\usetikzlibrary{shapes}

\icmltitlerunning{Robust Consensus Ranking}

\begin{document}

\twocolumn[
\icmltitle{Robust Consensus in Ranking Data Analysis:\\ Definitions, Properties and Computational Issues}



\icmlsetsymbol{equal}{*}

\begin{icmlauthorlist}
\icmlauthor{Morgane Goibert}{yyy,comp}
\icmlauthor{Clément Calauzènes}{yyy}
\icmlauthor{Ekhine Irurozki}{comp}
\icmlauthor{Stéphan Clémençon}{comp}

\end{icmlauthorlist}

\icmlaffiliation{yyy}{Criteo AI Lab, France}
\icmlaffiliation{comp}{Télécom Paris, France}

\icmlcorrespondingauthor{Morgane Goibert}{morgane.goibert@gmail.com}

\icmlkeywords{Machine Learning, ICML}

\vskip 0.3in
]



\printAffiliationsAndNotice{} 

\begin{abstract}

As the issue of robustness in AI systems becomes vital, statistical learning techniques that are reliable even in presence of partly contaminated data have to be developed. Preference data, in the form of (complete) rankings in the simplest situations, are no exception and the demand for appropriate concepts and tools is all the more pressing given that technologies fed by or producing this type of data (\textit{e.g.} search engines, recommending systems) are now massively deployed. However, the lack of vector space structure for the set of rankings (\textit{i.e.} the symmetric group $\mathfrak{S}_n$) and the complex nature of statistics considered in ranking data analysis make the formulation of robustness objectives in this domain challenging. In this paper, we introduce notions of robustness, together with dedicated statistical methods, for \textit{Consensus Ranking} the flagship problem in ranking data analysis, aiming at summarizing a probability distribution on $\mathfrak{S}_n$ by a \textit{median} ranking. Precisely, we propose specific extensions of the popular concept of \textit{breakdown point}, tailored to consensus ranking, and address the related computational issues. Beyond the theoretical contributions, the relevance of the approach proposed is supported by an experimental study.

\end{abstract}

\section{Introduction}

One of the keys to the path of a trustworthy AI is undeniably the design of statistical learning techniques that can resist, to a certain extent, possible corruptions of the training dataset. The analysis of the influence of atypical observations on the outputs of machine-learning algorithms has received increasing interest in the AI literature these last few years and has recently motivated a wide variety of dedicated works (refer to \citet{lugosi2019risk,monk} for instance), revisiting in particular seminal concepts in \textit{Robust Statistics} such as the $\varepsilon$-contamination model, where the training dataset is supposedly contaminated by a fraction $\varepsilon\in (0,1)$ of outliers \cite{huber1964robust}. It is the goal of this paper to investigate the statistical analysis of ranking data from the perspective of robustness. Ranking data are indeed ubiquitous in modern technologies such as search engines or recommending systems and the question of their reliability in presence of corrupted data is a scientific challenge. Given the nature of preference data, observable in the form of permutations (complete rankings, \textit{i.e.} elements of the symmetric group $\mathfrak{S}_n$) in the simplest case, informative statistics based on the latter are far from being simple. This is mainly due to the lack of vector space structure on $\mathfrak{S}_n$ and the impossibility of averaging directly such data. A major problem in ranking data analysis referred to as \textit{Consensus Ranking} or \textit{Ranking Aggregation}, and which the present article focuses on, consists in its simplest formulation in summarizing a ranking distribution (\textit{i.e.} a probability distribution on $\mathfrak{S}_n$) by a \textit{median ranking} \cite{Kemeny59}. Even though this problem has a long history in social choice theory, see \textit{e.g.} \citet{de2014essai,de1781memoire}, it has been the subject of much attention within the machine-learning community, see \textit{e.g.} \citet{procaccia2016optimal,JKS16} among many others, references being far too numerous to be listed exhaustively. While most documented works concern the issue of computing (approximately) median rankings with theoretical guarantees, this paper studies in contrast the robustness properties of consensus ranking methods by means of a novel approach, extending that developed in \citet{Huber} for multivariate data. We emphasize that this angle is original to the best of our knowledge and distinguishes itself from related results in social choice theory, where median rankings are identified with \textit{voting rules}. 
In line with these works, the well-known Gibbard-Satterthwaite theorem~\cite{gibbard1973manipulation, satterthwaite1975strategy}  states that every reasonable voting rule can be manipulated. We point out that there has been a wide body of research devoted to characterizing the complexity of computing manipulations, NP-hardness result on manipulation being considered as a guarantee for robustness \cite{bartholdi1989computational, Davies2011, Brandt2016}. However, beyond-worst-case analysis shows that the problems are easy in practice \cite{Zuckerman2009}. In the present article, we complement these works on the issue of robustness to vote manipulation by investigating how the seminal concept of \textit{breakdown point}, a popular measure of robustness of estimators in multivariate statistical analysis, may apply to consensus ranking. Basically, it can be defined as the proportion of outliers or (possibly deliberately) corrupted observations that can contaminate the data sample without jeopardizing the statistic. As will be shown, one of the main difficulties faced in the context considered here lies in the fact that consensus rankings are often obtained by solving an optimization problem and no closed analytical form for the solutions is available in general. Consequently, the computation of breakdown points of ranking statistics is generally a computational challenge. Our main proposal here consists in relaxing the constraint stipulating that the summary of a ranking distribution should be necessarily represented by a single ranking (\textit{i.e.} a strict order on the set of items indexed by $i\in \{1,\;\ldots,\; n \}$), or equivalently by a point mass on $\mathfrak{S}_n$. Instead, we suggest summarizing a ranking distribution by a \textit{bucket ranking} (\textit{i.e.} a weak order on the set $\{1,\;\ldots,\; n \}$), the possibility of observing ties in the orderings considered being shown to have crucial advantages regarding robustness.

The paper is organized as follows. In \cref{sec:setting}, basics in ranking aggregation and the notion of breakdown function are introduced, as well as the contributions of our paper. \cref{sec:robustness} focus on robustness, by detailing our theoretical results on the breakdown functions for the classical median, extending this concept to bucket rankings, and providing an optimization algorithm to estimate it in practice. \cref{sec:our_stats} is dedicated to the definition of our robust statistic, called the Downward Merge statistic. Finally, experiments are done in \cref{sec:exps} to highlight the usefulness of our Downward Merge statistic for solving Robust Consensus Ranking tasks.




\section{
Framework and Problem Statement}
\label{sec:setting}

We start with a reminder of key concepts in ranking data analysis and \textit{Robust Statistics}. The interested reader can refer to \citet{AY14,Huber} for more details. Here and throughout,
a ranking over a set of $n\geq 1$ items is represented as a permutation $\sigma \in \pS$ where $\pS$ is the symmetric group. By convention, the rank $r$ of an item $i\in[n]$ is $r=\sigma(i)$. For any measurable space $\cX$, $\cM^1_+(\cX)$ is the set of probability measures on $\cX$, ${\rm TV}(p,q)$ the total variation distance between $p$ and $q$ in $\cM^1_+(\cX)$.

\subsection{Ranking Data and Summary Statistics}

The descriptive analysis of probability distributions, or datasets for their empirical counterparts, is a fundamental problem in statistics. For distributions on Euclidean spaces such as $\lR^d$, this problem has been widely studied and covered by the literature, with the study of statistics ranging from the simplistic sample mean to more sophisticated data functionals, such as $U/L/R/M$-statistics or depth functions for instance \cite{vdV98}. 

Defining similar notions for probability distributions on $\pS$, the space of rankings, is challenging due to the absence of vector space structure. However, fueled by the recent surge of applications using preference data, such as \textit{e.g.} recommender systems, the statistical analysis of ranking data has recently regained attention and certain classic problems have been revisited, as for instance those related to consensus rankings and their generalization ability (see \textit{e.g.}~\citet{Korba2017} and the references therein) or to the extension of depth functions to ranking data \cite{goibert2022depthranking}.

\paragraph{Central tendency or location.} Statistics measuring centrality, such as the mean (or the median for univariate distribution), can be seen as barycenters of the sampling observations w.r.t a certain distance. Consensus Ranking / Ranking Aggregation extends this idea to probability distributions on $\pS$ \cite{Deza}. Given a (pseudo-)metric $d$ defined on $\pS$ and a distribution $p\in\distribs$, a \emph{ranking median} $\sigma^{\rm med}_{p,d}\in\pS$ can be defined as
\begin{align}
    \label{eq:ranking_median}
    \sigma^{\rm med}_{d}(p) := \argmin_{\sigma \in \pS} \lE_{\Sigma \sim p}(d(\sigma, \Sigma)).
\end{align}
A well-studied instance of ranking median is the \emph{Kemeny consensus}, which corresponds to the situation where $d$ is the \emph{Kendall Tau} distance: for all $\sigma,\;  \nu$ in $\pS$,
\begin{align}
    d_{\tau}(\sigma, \nu) = \frac{2}{n(n-1)}\sum_{i<j}\indicator{\sigma(i)<\sigma(j)}\indicator{\nu(i) > \nu(j)}
    \label{eq:kendall_tau}
\end{align}
Another common choice is the \emph{Borda count} when $d$ is the \emph{Spearman Rho}, see \cref{app:additional_metrics} for more details. 
Moreover, when $d$ is the Kendall tau, Borda is a $O(n \log n)$, 5-approximation of the Kemeny ranking~\cite{Caragiannis2013,JKS16,Coppersmith:2010}, which is a NP-hard to compute~\cite{Dwork}. 

\paragraph{More complex statistics based on ranking data.} Often, the information carried by a location statistic must be complemented. For instance, a notion of \emph{dispersion} or \emph{shape} is generally key to assessing convergence results or building confidence regions. To this end, the notion of \emph{statistical depth function} has been developed for multivariate data (in Euclidean spaces) (see \cite{ZuoSerfling00} and the references therein) and recently adapted to ranking, refer to \cite{goibert2022depthranking}.
However, as more complex statistics are more likely to exhibit robustness issues, we focus on simple statistics estimating location for ranking distribution.

\subsection{Robust Statistics}

To evaluate the robustness of a statistic, the notion of \textit{breakdown function} has been introduced in the seminal work of \cite{huber1964robust}. Informally, the breakdown function for a statistic $T$ on a distribution $p$ measures the minimal attack budget required for an adversarial distribution to change the outcome of the statistic $T$ by an amount at least $\delta > 0$. 



\begin{restatable}{definition}{defbreakdownfunction}\label{def:breakdown_function}
    {\sc{(Breakdown Function)}} Let $\cX$ and $\cY$ be measurable spaces, $p\in\cM^1_+(\cX)$, $T: \cM^1_+(\cX) \to \cY$ a measurable function and $d$ a metric on $\cY$.
    For any level $\delta \geq 0$, the breakdown function of the functional $T$ at $p$ is 
    \begin{align*}
        \varepsilon^\star_{d, p,T}(\delta) = \inf \left\{ \varepsilon > 0 \, \middle| \, \sup_{q : {\rm TV}(p,q) \leq \varepsilon} d(T(p), T(q)) \geq \delta \right\}.
    \end{align*}
\end{restatable}
In the traditional case $\cX=\cY=\mathbb{R}$, the level $\delta$ is generally set to $+\infty$ and the budget required is referred to as \emph{breakdown point}. 
In the extreme case, when $T$ is the identity and $\delta = 0^+$, $\varepsilon^\star$ quantifies the budget of attack under which \emph{identifiability} of the distribution is possible (which requires the additional knowledge that $p$ belongs to some family).

\paragraph{Application to Ranking Data.} In \citet{agarwal2020rank} such a study on identifiability is provided for the Bradley-Terry-Luce \cite{BT1952, Luce59} model under a budget constraint on pairwise marginals rather than the Total Variation, and \citet{jin2018} on the Heterogeneous Thurstone Models \cite{thurstone1927}.
However, summary statistics, such as a central tendency, are generally harder to break than the full distribution itself, so the breakdown function provides a finer quantification of robustness than the identifiability of the distribution.
Since the distances on $\pS$ are bounded, in general, the full breakdown function needs to be considered and one cannot focus only on a particular level such as $\delta = 0^+$ or $\delta = +\infty$. From here and throughout, the distance $d$ and the attack amplitude $\delta$ are normalized to lie between $0$ and $1$. 

The robustness of the median statistic when an adversary is allowed to attack with any strategy a pairwise model has also been studied ~\cite{datar2022byzantine}. They characterize the robustness of two statistics in terms of the L2 distance on distributions. We propose in Definition~\ref{def:breakdown_function} a more general and natural measure for robustness as a function of  the distance between the true and a corrupted statistic. 

\paragraph{Bucket Rankings as a robustness candidate.} In rankings, adversarial attacks often target pairs of items that are ``close" in some sense \cite{agarwal2020rank}: consecutive ranks, a pairwise marginal probability close to $\frac{1}{2}$, \ldots Thus, a simple and efficient way to robustify a ranking median is to accept \emph{ties}, rather than being restricted to a strict order.

\subsection{Challenges and Contributions}

There is a wide number of median statistic studies motivated by the lack of analytical expression and the computational and statistical challenges that arise in the estimation process.
However, robustness results for ranking statistics are rare and not rigorous enough for comparing different estimators. 

\begin{contribution}
Using \cref{def:breakdown_function} with the Kendall tau distance provides a straightforward measure of robustness for ranking medians.
In \cref{sec:bd_kem} we provide a lower-bound on the breakdown function for a ranking median (\cref{thm:ubbreakdownfunctionmedian}) and a tight upper-bound for the Kemeny consensus (\cref{thm:ubbreakdownfunctionmedian}). 
\end{contribution}

Moreover, slight perturbations in the pairwise relations of items that are similar to each other can imply breaking a median estimator, showing a lack of robustness. It is natural to propose more robust estimators by allowing pairs of items to be ``equally ranked", i.e., by considering bucket ranking statistics. However, generalizations of the breakdown function for bucket rankings require the use of Kendall tau for buckets, which is  computationally impractical. 

\begin{contribution}
In \cref{sec:buckets} we propose an extension of the breakdown function for bucket rankings which is built upon a Hausdorff generalization of the Kendall tau distance. We also develop an optimization algorithm to approximate this breakdown function that overcomes the computational issue of having a piece-wise constant objective function.
\end{contribution}


We illustrate and show empirically that bucket rankings are more robust median estimators than rankings. However, finding the optimal bucket order statistic requires exhaustively searching the space of bucket rankings $\wO$, which is even larger than the space of permutations, of factorial cardinality, and therefore, it is totally infeasible. 

\begin{contribution}
In \cref{sec:our_stats} we propose a general method for robustifying medians: given a ranking median, our algorithm successively merges ``similar'' items together into the same bucket. 
We evaluate this statistic in \cref{sec:exps}, showing an improvement of robustness w.r.t. Kemeny's median without sacrificing its precision. 
\end{contribution}

\section{Robustness - Breakdown Function for Ranking and Bucket Rankings}
\label{sec:robustness}

This section first details how to apply the notion of \textit{breakdown function} $\varepsilon^\star_{d, p, T}$. This allows providing insights into the robustness of classical location statistics such as the Kemeny consensus. These results advocate for the introduction of a more robust type of statistics based on bucket orders that are also developed in this section.

\subsection{Breakdown Function for the Kemeny Consensus}\label{sec:bd_kem}

We explore the robustness of ranking medians $\sigma^{\rm med}_{d}(p)$ as defined in \cref{eq:ranking_median} for different metrics $d$ over $\pS$ as defined by the breakdown function $\varepsilon^\star_{d_\tau, p, T}$. In particular, it is possible to tightly sandwich the breakdown function for the Kemeny median.

\begin{restatable}{theorem}{thmbreakdownfunctionkemeny}\label{thm:breakdownfunctionkemeny}
For $p\in\cM_+^1(\pS)$, ~ $\sigma^\star_p = \sigma^{\rm med}_{d_\tau}(p)$ (Kemeny median) and $\delta\geq0$, if $\varepsilon^+(\delta) \leq 2 p(\sigma_{p}^{*})$ then $\varepsilon^\star_{d_\tau, p,\sigma^\star_p}(\delta) \leq \varepsilon^+(\delta)$ with
    \begin{align*}
        \varepsilon^+(\delta) =
    \min_{\substack{\sigma \in\pS \\ d_\tau(\sigma, \sigma^\star_p) \geq \delta}} 
    \max_{\substack{\nu\in\pS \\ d_\tau(\nu, \sigma^\star_{p}) < \delta}} 
    \frac{\lE_{\Sigma\sim p}\left[d_\tau(\Sigma, \sigma) - d_\tau(\Sigma, \nu)\right]}
    {d_\tau(\sigma^\star_p, \sigma) - d_\tau(\sigma^\star_p, \nu) }\,.
    \end{align*}
\end{restatable}
\begin{proofsketch} Detailed Proof can be found in \cref{app:breakdown_function_kemeny_ub}
The proof relies on showing that, for $\varepsilon>0$, the \emph{attack} distribution $\bar{q}_\varepsilon = p - \frac{\varepsilon}{2}\indicator{\cdot = \sigma^*_p} + \frac{\varepsilon}{2}\indicator{\cdot = \sigma^{\star, {\rm rev}}_p}$, where $\sigma^{\star, {\rm rev}}_p$ is the reverse of $\sigma^\star_p$, is in the feasible set of the optimization problem $\sup_{q: \textsc{tv}(p,q)\leq \varepsilon}d_\tau(\sigma^*_p, \sigma^*_{q})$ (see \cref{def:breakdown_function}). 

Using $\bar{q}_\varepsilon$ provides a way to link $\varepsilon$ and $\delta$.
The condition $\varepsilon^+(\delta) \leq 2 p(\sigma^\star_p)$ ensures $\bar{q}_\varepsilon$ is well-defined.
\end{proofsketch}

It is also possible to provide a lower bound on the breakdown function for any generic ranking median.

\begin{restatable}{theorem}{thmubbreakdownfunctionmedian}\label{thm:ubbreakdownfunctionmedian}
For $p\in\cM_+^1(\pS)$, $m$ and $d$ being two metrics on $\pS$, ~ $\sigma^\star_p = \sigma^{\rm med}_{d}(p)$ and $\delta\geq0$, we have $\varepsilon^\star_{m, p,\sigma^\star_p}(\delta) \geq \varepsilon^-(\delta)$ with
    \begin{align*}
        \varepsilon^-(\delta) =
    \min_{\substack{\sigma \in\pS \\ m(\sigma, \sigma^\star_p) \geq \delta}} 
    \max_{\substack{\nu\in\pS \\ \nu \neq \sigma}} 
    \frac{\lE_{\Sigma\sim p}\left[d(\Sigma, \sigma) - d(\Sigma, \nu)\right]}
    {\max_{\sigma'\in\pS} d(\sigma', \sigma) - d(\sigma', \nu)}
    \end{align*}
\end{restatable}
\begin{proof}
Detailed proof can be found in \cref{app:breakdown_function_median_lb}.
\end{proof}

\begin{figure}[htbp]
    \centering
    \includegraphics[width=0.45\textwidth]{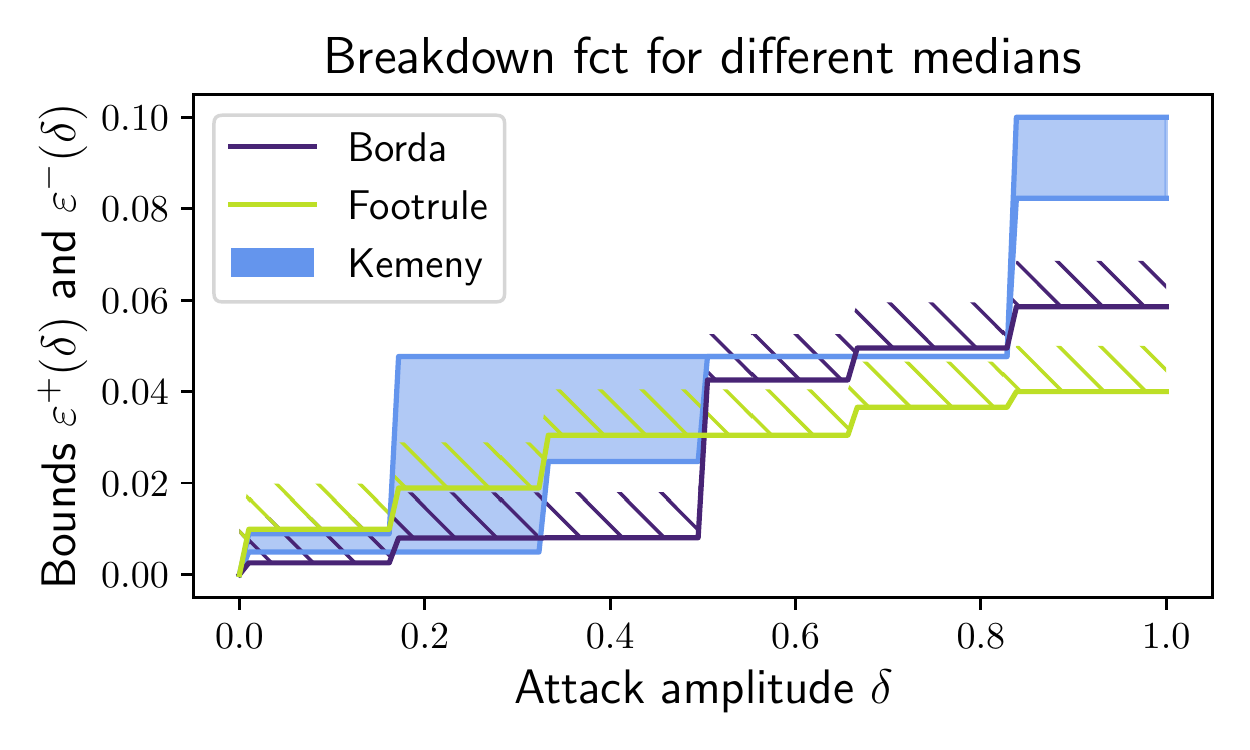}
    \caption{An illustration of $\varepsilon^+(\delta)$ and $\varepsilon^-(\delta)$ (from \cref{thm:breakdownfunctionkemeny} and \cref{thm:ubbreakdownfunctionmedian}) for a distribution on permutations of 4 items. For Borda and the median associated with Spearman footrule, only the lower bound is displayed.}
    \label{fig:breakdown_of_medians}
\end{figure}
\cref{fig:breakdown_of_medians} shows that no choice of $d$ makes the median uniformly more robust than another. Then, unfortunately, it also illustrates the fragility of median statistics against corruption of the distribution. In this example, impacting the distribution $p$ by less than $5\%$ allows changing the Kemeny median by flipping more than half item pairs ($\delta \geq 0.5$).

\paragraph{Sensitivity to similar items.} 
To further illustrate the fragility of Kemeny's median, \cref{fig:kemeny_on_indifference} shows its breakdown function on specific distributions. As could be expected, if all items are almost indifferent (uniform distribution - purple curve), then a ranking median is very fragile: a small nudge on $p$ is enough to change the Kemeny median from one ranking to its reverse. On the contrary, when $p$ is a point mass at a given ranking (blue curve), it requires a large attack on $p$ to impact the median. 

The green curve shows a weakness in the median: despite $p$ being concentrated on two neighbouring rankings (identical up to a pair of adjacent items), the robustness is very low for $\delta \leq 0.2$. This highlights a mechanism underlying adversarial attacks in real-world recommender systems (ex: fake reviews...): at a small cost, it is possible to be systematically ranked on top of close alternatives. This calls for using the natural alternative to (strict) rankings, which incorporates indifference between items: \emph{bucket rankings}.

\begin{figure}[htbp]
    \centering
    \includegraphics[width=0.45\textwidth]{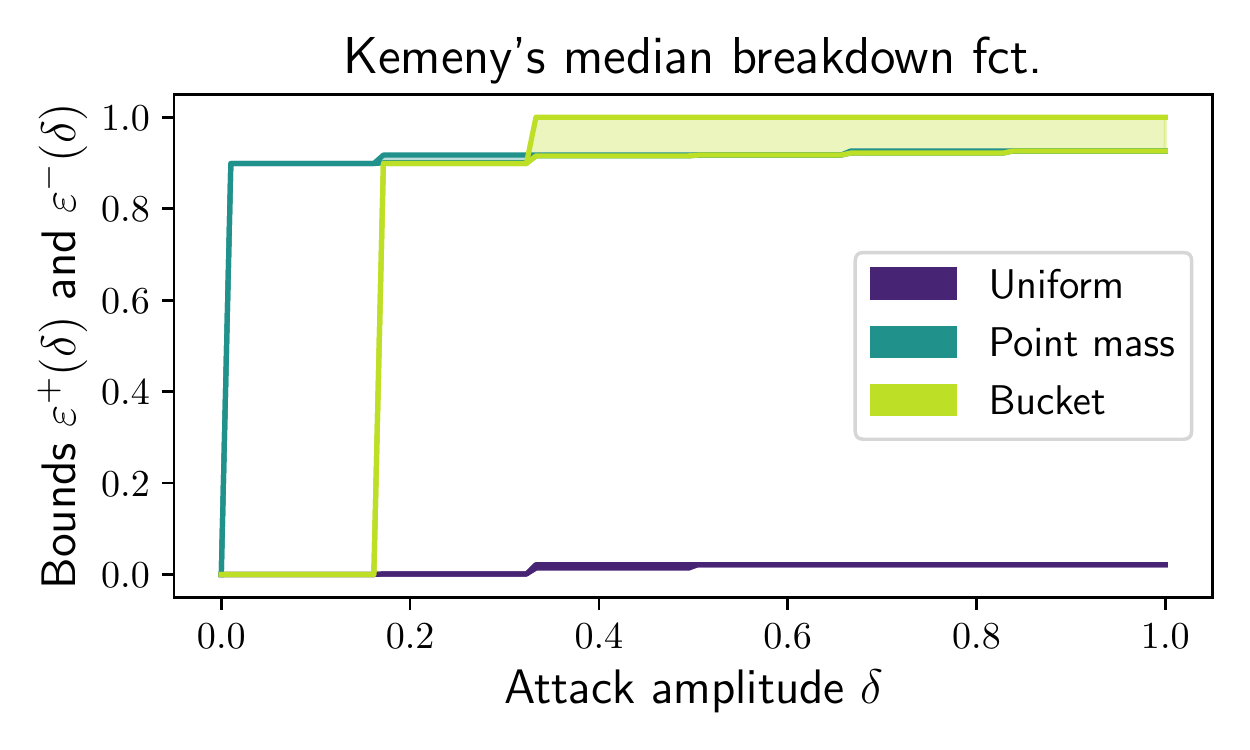}
    \caption{Breakdown function for Kemeny's median for different distributions $p$. "Uniform" denotes an almost uniform distribution; "Point mass" an almost point mass distribution, and "Bucket" an almost point mass distribution on two neighboring rankings.}
    \label{fig:kemeny_on_indifference}
\end{figure}

\subsection{Bucket Ranking - Extended Ranking Consensus}\label{sec:buckets}

Intuitively, bucket rankings are rankings with ties allowed. Formally, they can equivalently be defined as a total preorder -- \emph{i.e.} a homogeneous binary relation that satisfies transitivity and reflexivity (preorder) in which any two elements are comparable (total) -- or as a strict weak ordering -- \emph{i.e.} a strict total order over equivalence classes of items (buckets).

\begin{restatable}{definition}{defbucketrankings}\label{def:bucket_rankings}{\sc (Bucket ranking)}
A bucket order $\pi$ is a strict weak order defined by an ordered partition of $[n]$, \emph{i.e.} a sequence $(\pi^{(1)}, \dots , \pi^{(k)})$ of $k \geq 1$ pairwise disjoint non empty subsets (buckets) of $[n]$ such that: 
\begin{enumerate}
    \item[(i)] $i \prec_\pi j ~~\Leftrightarrow~~ \exists l<l' \in [k], (i,j) \in \pi^{(l)} \times \pi^{(l')}$,  
    \item[(ii)] $i \sim_\pi j ~~\Leftrightarrow~~ \exists l \in [k], (i,j) \in \pi^{(l)}\times\pi^{(l)}$,  
\end{enumerate}
We denote $\wO$ the set of bucket rankings, which is of size $\sum_{k=1}^n k! S(n,k)$\footnote{$S(n,k)$ are Stirling numbers of the second kind.} (vs $n!$ for $\pS$).
\end{restatable}
The indifference between items that bucket rankings can incorporate is an interesting feature to gain robustness, because the statistic can output alternatives between several strict orders, making it harder to attack.

\paragraph{As sets of permutations.} A bucket ranking $\pi\in\wO$ can be equivalently mapped to a subset of permutations, generated through the different ways to break ties. We say that a permutation $\sigma\in\pS$ is \emph{compatible} with a bucket ranking $\pi\in\wO$ -- denoted $\sigma\in\pi$ -- if for any $i,j\in[n]$, $\sigma(i)<\sigma(j) ~~\Leftrightarrow~~ i\prec_\pi j$ or $i\sim_\pi j$. For two bucket orders $\pi_1, \pi_2$, we say that $\pi_1$ is \emph{stricter} that $\pi_2$, denoted $\pi_1 \subseteq \pi_2$, iff for any $\sigma\in\pS, ~~\sigma\in\pi_1 \Rightarrow \sigma\in\pi_2$. 

\paragraph{As a distribution.} Being a set of permutations, a bucket order $\pi\in\wO$ can also be seen as a uniform distribution with restricted support. This point of view is particularly intuitive from a robustness perspective: a randomized output is generally harder to attack for an adversary.

\paragraph{Distances between bucket rankings.}
A key to applying the breakdown function from \cref{def:breakdown_function} to bucket orders statistics is to have a metric on $\wO$ that extends those defined on $\pS$.
To this end, we use the previous remark that weak orders are sets of rankings as well as a classical Hausdorff extension of metrics to sets. More precisely, we define:

\begin{restatable}{definition}{defnonsymmetrichausdorff}\label{def:non_symmetric_hausdorff}
    {\sc{(Non-symmetric Hausdorff)}} Let $d$ be a metric on $\pS$. The non-symmetric Hausdorff pseudoquasi-metric between two bucket rankings $\pi_1, \pi_2\in \wO$ is 
    \begin{align*}
        H^\text{\sc ns}_d(\pi_1, \pi_2) = \max_{\sigma_2 \in \pi_2} \min_{\sigma_1 \in \pi_1} d(\sigma_1, \sigma_2)  \,.      
    \end{align*}
\end{restatable}

Even though it is not a metric, $H^\text{\sc ns}_d$ is well-suited to ranking with ties. Intuitively, its lack of symmetry allows differentiating adversarial attacks whose effect is on the strict part of the bucket order (e.g. swapping two items that are strictly ordered) from those whose effect is "only" to disambiguate a tie. More precisely, if $\pi_2 \subseteq \pi_1$, then $H^\text{\sc ns}_d(\pi_1, \pi_2) = 0$. Depending on the application, one may want to focus on the first type of attacks, in which case $H^\text{\sc ns}_d$ is a suitable choice to define the breakdown function as $\varepsilon^\star_{H^\text{\sc ns}_d, p, T}$.
Otherwise, it is possible (and usual) to symmetrize the Hausdorff metric.
\begin{restatable}{definition}{defsymmetrichausdorff}\label{def:symmetric_hausdorff}
    {\sc{($1/2$-symmetric Hausdorff)}} Let $d$ be a metric on $\pS$. The $1/2$-symmetric Hausdorff metric between two bucket rankings $\pi_1, \pi_2\in \wO$ is defined by
    \begin{align*}
        H^{(1/2)}_d(\pi_1, \pi_2) = \frac{1}{2} \Big(H^\text{\sc ns}_d(\pi_1, \pi_2) + H^\text{\sc ns}_d(\pi_2, \pi_1)\Big)\,.
    \end{align*}
\end{restatable}
Usual symmetrization of the Hausdorff metric uses a maximum rather than an average \cite{fagin2006comparing}. However, under the Kendall-tau distance, the average version is computationally simpler (see \cref{app:hausdorff_kendall} for more details).

\subsection{The Breakdown Function in Ranking Data Analysis - Definition and Estimation}

\paragraph{Definition.}
Putting all the pieces together, from now on, the statistic $T : \cM_+^1(\pS) \to \wO$ 
 summarizes a distribution over $\pS$ by a bucket ranking in $\wO$. Then, we use either $H^{(NS)}_{d_\tau}(\pi_1, \pi_2)$ (see \cref{def:non_symmetric_hausdorff}) or $H^{(1/2)}_{d_\tau}(\pi_1, \pi_2)$ on $\wO$ where $d_\tau$ is the Kendall tau (see \cref{eq:kendall_tau}). Finally, the breakdown function $\varepsilon^\star_{H^{(NS)}_{d_\tau}, p, T}$ is the result of the following optimization problem
 \begin{align}
     \inf \left\{ \varepsilon > 0 \, \middle| \, \sup_{q : {\rm TV}(p,q) \leq \varepsilon} H^{(NS)}_{d_\tau}(T(p), T(q)) \geq \delta \right\}
     \label{eq:breakdown_fct_bucket}
 \end{align}

\paragraph{The Empirical Breakdown Function.}
Computing a closed-form expression for the breakdown point for any statistic $T$ and distribution $p$ is challenging in general. However, it can be estimated empirically: the extended expression of the breakdown function in \cref{eq:breakdown_fct_bucket} can be simplified so that it is the solution to the following Lagrangian-relaxed optimization problem.
\begin{equation}
    \inf_{q \in \Delta^{\frak{S}_n}} \sup_{\lambda \geq 0} 1/2 \| p-q \|_1 + \lambda(\delta - H^{(NS)}_{d_\tau}(T(p), T(q)) )
    \label{eq:lagrangian_relax_bkdwn}
\end{equation}
\paragraph{Smoothing.} As $H^{(NS)}_{d_\tau}(T(p), T(q)))$ is piece-wise constant as a function of $q$ (with a combinatorial number of pieces), Problem \eqref{eq:lagrangian_relax_bkdwn} cannot directly be solve using standard optimization techniques.
To solve this issue, we used a smoothing procedure by convolving this function with a smoothing kernel $k_\gamma$ with scale $\gamma$. Thus, after the relaxation, the optimization problem \eqref{eq:lagrangian_relax_bkdwn} becomes:
\begin{equation}
    \inf_{q \in \Delta^{\frak{S}_n}} \sup_{\lambda \geq 0} 1/2 \| p-q \|_1 + \lambda(\delta - \rho_T(p, q) ),
    \label{eq:smoothing_version_bkdwn}
\end{equation}
with 
\begin{equation}
    \begin{split}
        \rho_T(p,q) &= H^{(NS)}_{d_\tau}(T(p), T(q)) \star k_{\gamma}(q) \\
        &= \int_{u} H^{(NS)}_{d_\tau}(T(p), T(u)) \times k_{\gamma}(q-u) \text{d}u,
    \end{split}
\end{equation}
On a practical note, a simple way to build a convolution kernel $k_\gamma$ on a simplex like $\distribs$, is to use a convolution kernel $\kappa_\gamma$ on the whole euclidean space -- for instance an independent Gaussian density $\kappa_{\gamma}(x) = \frac{1}{\sqrt{ (2 \pi \gamma)^{n!} }} \exp{( -\frac{x^{\text{T}} x }{2 \gamma^2} )}$ 
-- and set $k_\gamma$ to be the density of the push-forward through a \emph{softmax} function. We denote $\varepsilon^\gamma_{p,T}(\delta)$ the limiting value of $\|p-q\|_1/2$ at the solution of \eqref{eq:smoothing_version_bkdwn}. Note the bias induced by such definition of $k_\gamma$ fades away when $\gamma$ goes to $0$ in the same way as the bias induced by the convolution.
This smoothing ensures $\rho_T$ is a continuous, differentiable function with respect to $q$. Moreover, it can easily be estimated using a Monte-Carlo sampling, using the following remark: $\rho_T(p,q) = \mathbb{E}_{u \sim k_{(p, \gamma)}} (H^{(NS)}_{d_\tau}(T(u), T(q))$.

\paragraph{Optimization.} When using Monte-Carlo estimation for $\rho_T$, \cref{eq:smoothing_version_bkdwn} is a stochastic saddle-point problem. To solve such problems, gradient/ascent has a rate of convergence of $\cO(t^{1/2})$ for its ergodic average ($t$ being the number of steps) \cite{Nemirovski2002}. 
Our empirical optimization algorithm for computing the breakdown functions relies on stochastic gradient descent and is able to provide good approximations, as illustrated in \cref{fig:theory_exp_kemeny_maxpair}. We denote $\hat{\varepsilon}^\gamma_{p,T}(\delta) = \|p - \bar{q}_t\|_1$, where $\bar{q}_t$ is the ergodic average of the iterates $(q_s)_{s\leq t}$ obtained during the optimization.

Let's make a couple of remarks on the empirical breakdown function $\hat{\varepsilon}^\gamma_{p,T}$. First, it is a noisy estimate of $\varepsilon^\gamma_{p, T}$ as $\rho_T$ and its gradients are estimated via Monte-Carlo. Thus, the choice of $\gamma$ and $t$ should trade-off the variance of $\hat{\varepsilon}^\gamma_{p,T}$ and the bias $|\varepsilon^\gamma_{p, T}-\varepsilon^\star_{d_\tau, p, T}|$. Second, as the term $\|p - q\|_1$ is minimized in \eqref{eq:smoothing_version_bkdwn}, it is expected $\hat{\varepsilon}^\gamma_{p,T}$ over-estimates $\varepsilon^\gamma_{p, T}$.

\section{Robust Consensus Ranking Statistics}
\label{sec:our_stats}

As proved by \cref{thm:breakdownfunctionkemeny}, the classical median statistics as defined by \eqref{eq:ranking_median} can be easily broken, which motivates defining more robust statistics, based on bucket rankings. As illustrated by \cref{fig:kemeny_on_indifference}, the weakness of median statistics comes from being ``forced" to rank all items, even those which are (almost) indistinguishable. Bucket rankings seem to be a natural solution to this problem, but \emph{what is a good way to build a bucket order statistic?}

As $H^{(NS)}_{d_\tau}$ defines a (pseudoquasi-) distance on $\wO$, we could adapt the idea of a median as in \eqref{eq:ranking_median} for bucket rankings. However, contrarily Borda medians which can be computed in a scalable way \cite{Caragiannis2013}, Hausdorff-based medians would require to optimize over $\wO$. As its cardinality is larger than $\frak{S}_n$ this problem can be more computationally challenging than Kemeny's median.

A more scalable approach is to start from a ranking median such as the Kemeny or Borda consensus and to robustify it using a plug-in method based on merging items that are close into buckets. \cref{fig:merge_limit} illustrates this idea. The left graph describes pairwise marginal probabilities for which the Kemeny consensus is $A\prec B\prec C\prec D$. Intuitively, merging either $C$ and $D$ (as $\lP(C \prec D) = 0.51)$ or $B$ and $C$ (as $\lP(B \prec C) = 0.52)$ leads to bucket rankings (i) and (ii), which will be harder to attack. However, this example also highlights that there is no unique way of merging items. For instance, if the constraint is to only merge items whose pairwise preference probability is in $[0.4, 0.6]$, it is possible to merge $B,C$ or $C,D$, but not $B,C,D$ as $\lP(B\prec D) = 0.7$: \emph{pairwise indistinguishability is not transitive}.

\begin{figure*}
\centering
\input{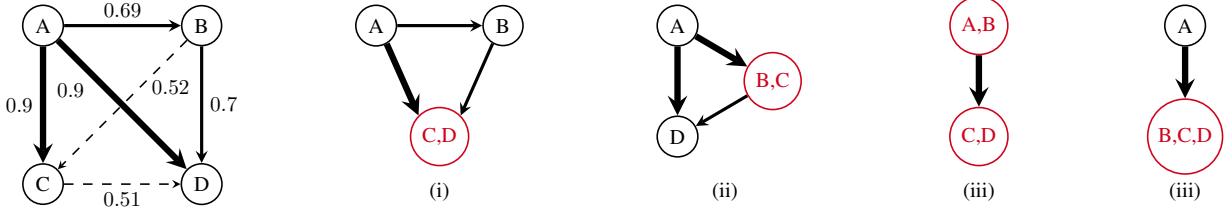}        
\caption{Left: Directed Graph that summarizes a pairwise marginal probability matrix. (i-iv) Graph representations of bucket orders that are compatible with merging items which pairwise preference probability is below 0.52 (i, ii) and below 0.7 (iii,iv).}
\label{fig:merge_limit}
\end{figure*}

\subsection{Na\"ive Merge Statistic}

In order to formalize the latter intuition and to derive a first (na\"ive) plug-in rule, we define the pairwise preference probability between two items, which provides a relevant notion of closeness between items.
\begin{definition}
    {\sc (Pairwise probabilities)}. For $p\in\cM_+^1(\pS)$, the pairwise preference probability between items $i \text{ and } j$, denoted $P_{i,j}$, is defined for $i\neq j$ by: $P_{i,j} = \mathbb{P}_{\Sigma \sim p}(\Sigma(i) < \Sigma(j))$. By convention, $P_{ii} = 0.5$. We define the pairwise matrix of $p$ as $P:= [P_{i,j}]_{1 \leq i,j \leq n}$.
    \label{def:pairwise_matrix}
\end{definition}
Then, given a bucket ranking $\pi\in\wO$, we  formalize the notion that two buckets can be merged, with the constraint of not changing the strict order between buckets. To this end, we define $\bar{P}_{i}(\pi)$, the \emph{strongest deviation from indifference} between any two items within the $i^{\rm th}$ bucket $\pi^{(i)}$.
\begin{align}
\bar{P}_{i}(\pi) = \max \left\{\left|P_{l,l'} - 0.5\right| : (l,l')\in \pi^{(i)}\right\}
\end{align}
Then, one needs to quantify the value of $\bar{P}_{i}(\pi)$ that would result from merging bucket $i$ to bucket $j$,
\begin{align}
\bar{P}_{ij}(\pi) = \max \left\{\left|P_{l,l'} - \frac{1}{2}\right| : (l,l')\in \bigcup_{\substack{l\in[n]\\i\leq l \leq j}}\pi^{(l)}\right\}
\end{align}
Finally, given a threshold $\theta\in[0,0.5]$ on the acceptable deviation from indifference, we define the set of pairs of buckets that can be merged while keeping $\bar{P}$ below $\theta$,
\begin{align}
    \cG(\pi,\theta) = \left\{(i,j)\in [n]^2: \bar{P}_{ij}(\pi) \leq \theta\right\}
\end{align}
The first intuition is to merge buckets iteratively, starting with the most indifferent ones, as described in \cref{algo_naive_merge}.

\begin{algorithm}
\DontPrintSemicolon
\SetKwInOut{Input}{Input}
\SetKwInOut{Output}{Output}
\Input{Pairwise matrix $P$, Ranking median $\sigma$, threshold $\theta \in [0, 0.5]$.}
$\pi \gets \sigma$ \tcp*{$\sigma$ as a bucket ranking}
\While{$\cG(\pi, \theta) \neq \emptyset$}{
    $(i^*, j^*) = \argmin_{(i,j)\in\cG(\pi,\theta)} \bar{P}_{ij}(\pi)$ \;
    update $\pi$ by merging all buckets between $i^*$ and $j^*$
    \vskip -2em
    \begin{flushleft}
        \begin{flalign*} 
            \begin{cases}
                \pi^{(i)} &\gets \pi^{(i)} ~~~\text{for}~ i < i^*\\
                \pi^{(i^*)} & \gets \bigcup_{l\in[n], i^*\leq l\leq j^*}\pi^{(l)}\\
                \pi^{(i - j^* + i^*)} & \gets \pi^{(i)} ~~~\text{for}~ i > j^*
            \end{cases}&&
        \end{flalign*}
    \end{flushleft}
    }
    \Output{$\pi$}
\caption{Na\"ive Merge}
\label{algo_naive_merge}
\end{algorithm}

Termination of \cref{algo_naive_merge} is guaranteed by the fact that the number of buckets in $\pi$ strictly decreases at each iteration. Then, by definition of $\cG(\pi, \theta)$, the resulting bucket ranking $\pi$ is such that any of its bucket $i$ satisfies $\bar{P}_i(\pi) \leq \theta$ -- \emph{i.e.} no two items with higher deviation than $\theta$ have been merged.

Despite being very natural, this algorithm suffers from an important limitation: when changing the threshold $\theta$, its output only spans a limited subset of valid bucket rankings. In the example provided by \cref{fig:merge_limit}, the na\"ive merge method plugged-in on the Kemeny consensus can only output (i) and (iii). Whatever the value of $\theta$, it can never output (ii) or (iv). This limitation is induced by its outputs being a monotonic (w.r.t. to inclusion) function of $\theta$ -- \emph{i.e.} for $\theta_1 \leq \theta_2$, the resulting bucket rankings satisfy $\pi_{\theta_1} \subseteq \pi_{\theta_2}$.

\subsection{Downward Merge Statistic}

Overcoming this limitation only requires a small change in the algorithm which results in our main plug-in method named \emph{Downward Merge}, shown in \cref{algo_downward_merge}. 
Downward Merge algorithm selects the two buckets $(i^*, j^*)$ whose deviation from indifference $\bar{P}_{ij}(\pi)$ is maximal among those $\bar{P}_{ij}(\pi)\leq \theta$. \footnote{Instead of taking the most similar buckets, as in the previous statistic, we take the most different pair among those that are ``similar enough".} Then, all the buckets $l$ such that $i^*\leq l\leq j^*$ are merged. This process is repeated while there exist pairs of buckets whose deviation from indifference $\bar{P}_{ij}(\pi)\leq \theta$ and thus termination is guaranteed.

\begin{algorithm}
\DontPrintSemicolon
\SetKwInOut{Input}{Input}
\SetKwInOut{Output}{Output}
\Input{Pairwise matrix $P$, Ranking median $\sigma$, threshold $t \in [0, 0.5]$.}
$\pi \gets \sigma$ \tcp*{$\sigma$ as a bucket ranking}
\While{$\cG(\pi, t) \neq \emptyset$}{
    $(i^*, j^*) = \argmax_{(i,j)\in\cG(\pi,t)} \bar{P}_{ij}(\pi)$ \;
    update $\pi$ by merging all buckets between $i^*$ and $j^*$
    \vskip -2em
    \begin{flushleft}
        \begin{flalign*} 
            \begin{cases}
                \pi^{(i)} &\gets \pi^{(i)} ~~~\text{for}~ i < i^*\\
                \pi^{(i^*)} & \gets \bigcup_{l\in[n], i^*\leq l\leq j^*}\pi^{(l)}\\
                \pi^{(i - j^* + i^*)} & \gets \pi^{(i)} ~~~\text{for}~ i > j^*
            \end{cases}&&
        \end{flalign*}
    \end{flushleft}
    }
    \Output{$\pi$}
\caption{Downward Merge}
\label{algo_downward_merge}
\end{algorithm}

The Downward Merge method is thus able to span a larger set of bucket orders when varying $\theta$. In the example from \cref{fig:merge_limit}, the Downward Merge method plugged-in on the Kemeny consensus can generate all four bucket rankings (i-iv) for $\theta\in \{0.51, 0.52, 0.69, 0.7)\}$.

The next experimental section illustrates the robustness improvement brought by this plug-in method over a ranking median.

\section{Numerical Experiments}
\label{sec:exps}

In this section, we illustrate the relevance of the statistic outputted by our Downward Merge plug-in on Kemeny's median (called our \textit{Downward Merge statistic} for short) by running several illustrative experiments for various settings and comparing with the baseline provided by the usual Kemeny's median. The code is available \href{https://github.com/RobustConsensusRanking/RobustConsensusRanking}{here}. 

\subsection{Empirical Robustness}
\label{subsec:emp_rob}

\begin{figure}
\centering
\includegraphics[width=\linewidth]{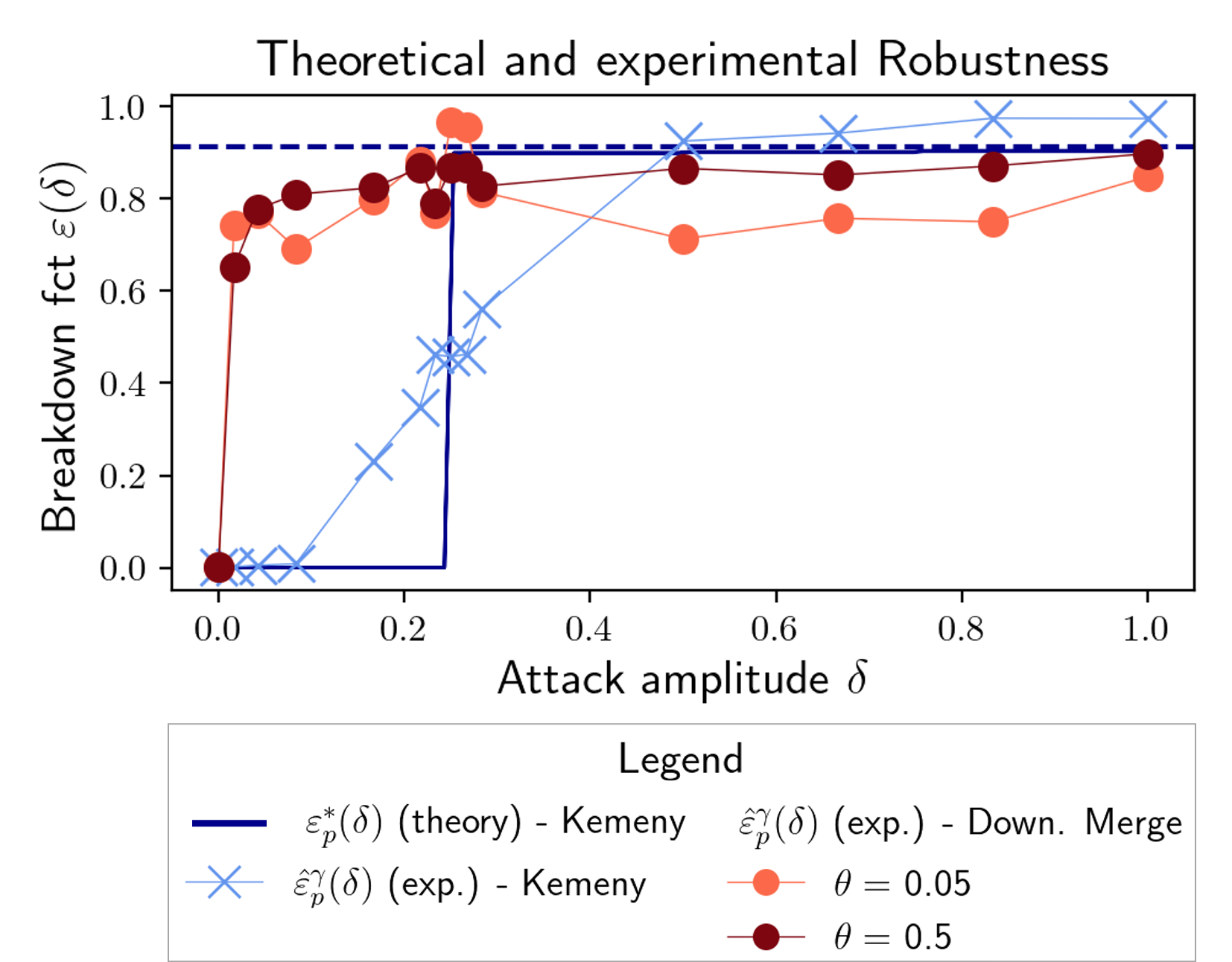}
    \caption{Breakdown function $\hat{\varepsilon}^{\gamma}_{p,T}(\delta)$ as a function of attack amplitude $\delta$ for a bucket distribution $p$ (almost a point mass on two neighboring rankings) with $n=4$. The plain blue line denotes the theoretical value for Kemeny's median $\varepsilon^*_{p}(\delta)$, blue crosses (resp. red dots) the empirical approximation $\hat{\varepsilon}^{\gamma}_{p,T}$ for Kemeny's median (resp. Down. Merge statistic for different thresholds $\theta$).}
    \label{fig:theory_exp_kemeny_maxpair}
\end{figure}

Our Downward Merge plug-in aims at providing a robustified statistic. To illustrate its usefulness, we ran experiments computing the approximate breakdown functions $\hat{\varepsilon}^{\gamma}_{p,T}(\delta)$ for the Kemeny's median as a baseline and our statistic when varying $\delta$. \cref{fig:theory_exp_kemeny_maxpair} shows the robustness as a function of attack amplitude $\delta$ and for a hand-picked distribution $p$ that is almost a point mass on a bucket ranking.


When the threshold is set to a sensible value (here $\theta = 0.05$), the Downward Merge algorithm outputs a bucket order as a statistic: thus, the robustness increases very strongly to reach nearly optimal values even for very small values of $\delta$, which illustrates its efficiency. When $\theta=0.5$, the statistic is the bucket order regrouping all items. In this case, the statistic cannot be broken, and provide optimal values for the breakdown function. However, such a statistic does not provide any information about the distribution under analysis: its accuracy of location is very poor. Formally, the accuracy of location of a statistic $T$ is defined by its closeness (under the same metric $d$ used in its definition) to the whole ranking distribution: $AL_{d, p}(T) := \| d \|_{\infty} -\mathbb{E}_{p}(d(T(p), \Sigma))$, which is the opposite of the \textit{loss}, as simply defined by $Loss_{d,p}(T) = \mathbb{E}_{p}(d(T(p), \Sigma))$. By definition, under metric $d = d_{\tau}$, Kemeny's median has the highest accuracy of location, \textit{i.e.} the smallest loss. On the other hand, the Downward Merge statistic when $\theta=0.5$ has a very high loss, which makes it irrelevant in most cases. These observations justify the analysis of the loss/robustness tradeoff of our Downward Merge statistic compared to Kemeny's median.

\subsection{Tradeoffs between Loss and Robustness}
\label{subsec:tradeoffs_classic}

\begin{figure}
\centering
    \includegraphics[width=\linewidth]{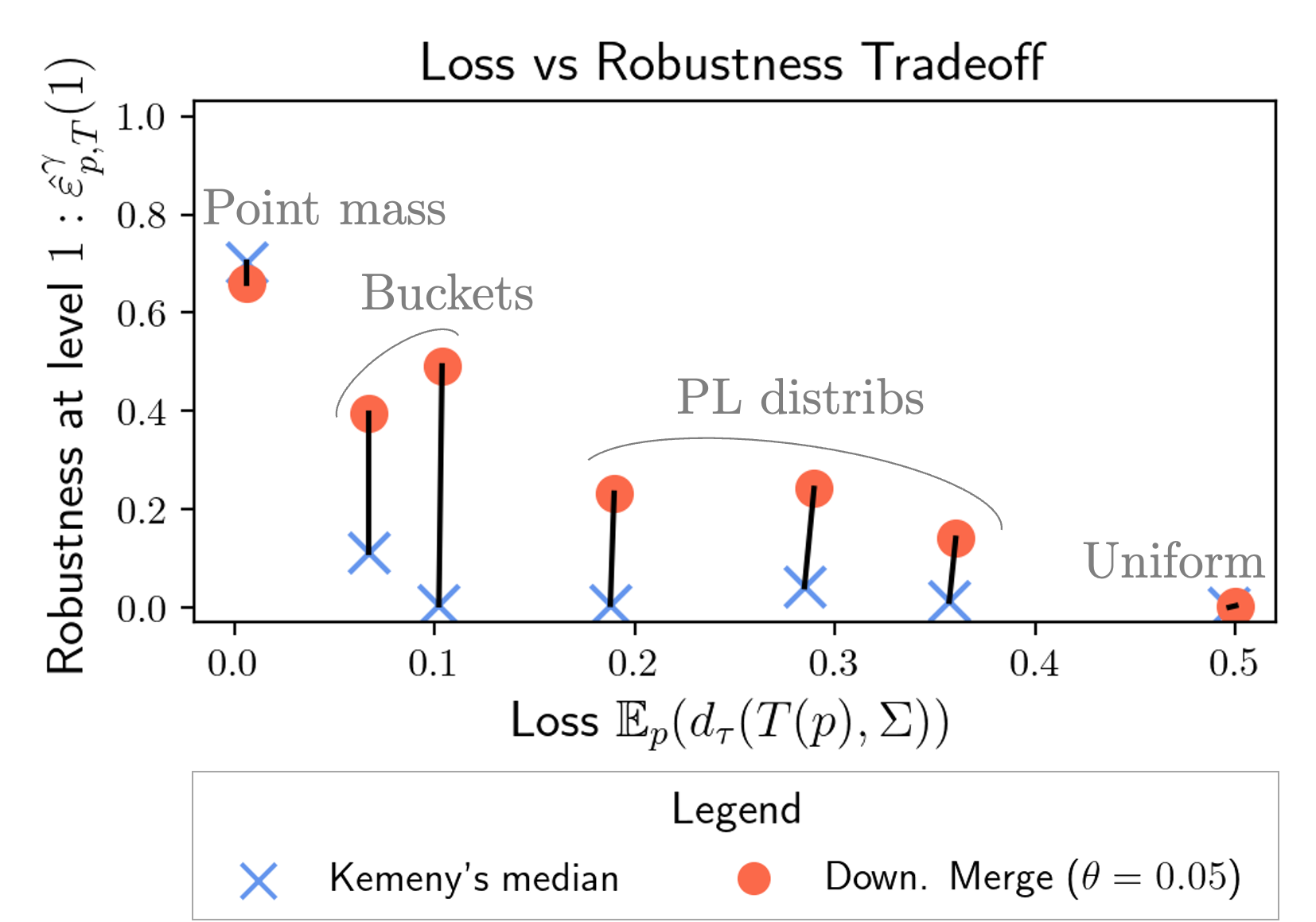}
\caption{Loss/Robustness tradeoffs for different $p$ with $\delta=1$. Pairs of points linked by a black line denote results for Kemeny's median and Down. Merge statistics on the same distribution $p$ with $n=4$. "Buckets" are hand-picked distributions generated to be almost a point mass on a bucket order, "Uniform" (resp. "Point mass") "is an almost uniform (resp. point mass) hand-picked distribution, and "PL distribs." are random Plackett-Luce distributions.}
\label{fig:res_classic}
\end{figure}

We ran experiments for various distributions $p$ and computed the loss and the breakdown function of Kemeny's median and our Downward Merge algorithm to show the loss/robustness tradeoff for each statistic. \cref{fig:res_classic} shows the results for different choices of distribution $p$ when the number of items $n=4$, and for $\delta=1/6$ (normalized value of $\delta$ that requires at least a switch between two items to break the statistic).

The point mass (resp. the uniform) distribution represents the extreme case for which Kemeny's median is very robust (resp. not robust at all) and for which we expect no improvement from using the Downward Merge statistic. This intuition is verified in both cases, and we can see that the Downward Merge statistic yields the same results (in loss and in robustness) as Kemeny's median.

The bucket distributions (for which the gap between the probabilities for two rankings in the bucket order is respectively $0.1$ and $0.01$) represent the settings to which our Downward Merge is best suited. As expected, the improvement in robustness when using our Downward Merge statistic is high, and the increase in loss is negligible.

Finally, the Plackett Luce distributions (for which the parameters were generated randomly) represent a random setting. The results are interestingly very similar to those for the bucket distributions: the gain in robustness is high and the increase in loss is negligible. This random setting illustrates the usefulness of our Downward Merge statistic in general cases and shows that, overall, it yields a much better compromise than Kemeny's median.

\section{Conclusion}

In this paper, we developed a framework to study robustness in ranks: we defined breakdown functions for rankings, extended it to bucket rankings, and created an optimization algorithm to approximate its value in practice. We developed our Downward Merge statistic as a plug-in to the classical Kemeny's median to provide, as confirmed by our experiments, not only an improved robustness but also a better compromise between centrality and robustness. We ensured our Downward Merge algorithm is scalable to practical settings, but the evaluation of the breakdown function remains challenging because of the use of the Total-Variation distance as a metric for the budget constraint. The definition and study of further scalable approximations of the breakdown function are left for future work.


\bibliography{icml_main}
\bibliographystyle{icml2023}

\newpage
\appendix
\onecolumn

\section{Additional Metrics on $\pS$}
\label{app:additional_metrics}

\paragraph{The Kendall Tau} is the metric used all along the main part of the paper, the proportion of misordered pairs,
\begin{align*}
    d_{\tau}(\sigma, \nu) = \frac{2}{n(n-1)}\sum_{i<j}\indicator{\sigma(i)<\sigma(j)}\indicator{\nu(i) > \nu(j)}\,.
\end{align*}
The Kemeny consensus is the median associated with the Kendall Tau metric.

\paragraph{The Spearman Rho} is a normalized quadratic distance between the rank vectors,
\begin{align}
    d_{\tau}(\sigma, \nu) = \frac{6}{n(n^2-1)}\sum_{i}\left(\nu(i) - \sigma(i)\right)^2\,.
\end{align}
The Borda count is the median associated with the Spearman Rho (\emph{e.g.} see \citet{calauzenes2013}).

\paragraph{The Spearman footrule} is a absolute value distance between the rank vectors,
\begin{align}
    d_{\tau}(\sigma, \nu) = \sum_{i}\left|\nu(i) - \sigma(i)\right|\,.
\end{align}

\section{Notation for Appendix}
\label{app:notation}
For the sake of clarity of the proofs, we switch to matrix notation in the appendix.
We fix an arbitrary indexation $\{\sigma^{(1)}, \dots, \sigma^{(n!)}\}$ of $\pS$. Using this indexation, given a metric $d$ on $\pS$, we can defined the (symmetric) metric matrix $D = (d(\sigma^{(i)},\sigma^{(j)}))_{i,j\in[n!]}$. Identifying a ranking $\sigma$ with its corresponding basis vector $\be_i$ s.t. $\sigma = \sigma^{(i)}$, we write for two rankings $\sigma, \sigma', \nu\in\pS$,
\begin{align}
    \nu^\top D \sigma := d(\nu, \sigma) ~~~~~~~\text{or}~~~~~~~ \nu^\top D (\sigma - \sigma') := d(\nu, \sigma) - d(\nu, \sigma')
\end{align}
Further, a distribution $p\in\cM_+^1(\pS)$ on permutation can now be seen as a $n!$-dimensional vector in $\mathbb{R}^{n!}$. This allows to write, for $p\in\distribs$, $\sigma\in\pS$,
\begin{align}
    p^\top D\sigma := \lE_{\Sigma\sim p}[d(\Sigma,\sigma)]
\end{align}

\section{Proof: Bound on Breakdown Function for Ranking Medians}
\label{app:breakdown_function_medians}

\subsection{Upper-bound}
\label{app:breakdown_function_kemeny_ub}

We first remind \cref{thm:breakdownfunctionkemeny}.
\thmbreakdownfunctionkemeny*

We re-state the theorem with the matrix notation defined in \cref{app:notation} and used all along the appendix.

\begin{theorem}
For $p\in\cM_+^1(\pS)$, $\sigma^\star_p = \sigma^{\rm med}_{d_\tau}(p)$ and $S_\delta = \{\sigma\in\pS | d_\tau(\sigma, \sigma^\star_p) \geq \delta\}$, if $\varepsilon^+(\delta) \leq 2 p(\sigma^\star_p)$, then $\varepsilon^\star_{d_\tau, p, \sigma^\star_p} \leq \varepsilon^+(\delta)$.
\begin{align}
    \varepsilon^+(\delta)
    =
    \min_{\sigma\in S_\delta}\max_{\nu\in N_\delta}
    \frac{\mkendall{(\sigma-\nu)}{p}}
    {\mkendall{(\sigma - \nu)}{\sigma^\star_p}}\,,
\end{align}

\end{theorem}
\begin{proof}
    \begin{align}
        \varepsilon^\star_{d_\tau, p, \sigma^\star_p} & = \inf \left\{\varepsilon>0\middle|\sup_{q: \textsc{tv}(p, q)\leq \varepsilon}d_\tau(\sigma^\star_p, \sigma^\star_q) \geq \delta\right\} \\
        & = \inf \left\{\varepsilon>0\middle|\exists q, s.t. \textsc{tv}(p, q)\leq \varepsilon ~\text{and}~ d_\tau(\sigma^\star_p, \sigma^\star_q) \geq \delta\right\}\\
        & = \inf \underbrace{\left\{\varepsilon>0\middle|\exists q, s.t. \textsc{tv}(p, q)\leq \varepsilon ~\text{and}~ \argmin_{\sigma\in\pS}\mkendall{\sigma}{q} \subseteq S_\delta\right\}}_{=: E} ~~~\text{with}~S_\delta = \{\sigma\in\pS | d_\tau(\sigma, \sigma^\star_p) \geq \delta\}
    \end{align}
    Further, we define 
    $N_\delta = \pS \setminus S_\delta$, 
    $\sigma^{\star, {\rm rev}}_p$ the reverse of $\sigma^{\star}_p$, i.e., $\sigma^{\star, {\rm rev}}_p (i)=\sigma^{\star}_p(n-i-1)$ and the \emph{attack} distribution 
    ${\bar{q}_\varepsilon = p - \frac{\varepsilon}{2}\indicator{\cdot = \sigma^\star_p} + \frac{\varepsilon}{2}\indicator{\cdot = \sigma^{\star, {\rm rev}}_p}}$ that removes the probability mass from the median to put it on the farthest point. 
    We also define 
    ${{E} = \left\{\varepsilon | \argmin_{\sigma\in\pS}\mkendall{\sigma}{\bar{q}_\varepsilon} \subseteq S_\delta\right\}}$  and 
    ${\tilde{E} = \left\{0<\varepsilon \leq 2 p(\sigma^\star_p) \middle | \argmin_{\sigma\in\pS}\mkendall{\sigma}{\bar{q}_\varepsilon} \subseteq S_\delta\right\} \subseteq E \cap (0, 2 p(\sigma^\star_p)]}$.

    Let $\varepsilon>0$ be such that $\varepsilon \leq 2 p(\sigma^\star_p)$. Then
    \begin{align}
        \varepsilon\in\tilde{E} 
        & \Leftrightarrow \exists \sigma \in S_\delta, \forall \nu \in N_\delta, \mkendall{\sigma}{\bar{q}_\varepsilon} \leq \mkendall{\nu}{\bar{q}_\varepsilon}\\
        & \Leftrightarrow \exists \sigma \in S_\delta, \forall \nu \in N_\delta, \mkendall{(\sigma - \nu)}{p} + \frac{\varepsilon}{2}\left(\mkendall{\sigma^{\star, {\rm rev}}_p}{\sigma} - \mkendall{\sigma^{\star}_p}{\sigma} + \mkendall{\sigma^{\star}_p}{\nu} - \mkendall{\sigma^{\star, {\rm rev}}_p}{\nu}\right)\leq 0\\
        & \Leftrightarrow \exists \sigma \in S_\delta, \forall \nu \in N_\delta, \mkendall{(\sigma - \nu)}{p} 
        \leq 
        \frac{\varepsilon}{2}\left(
             \mkendall{(\sigma - \nu)}{(\sigma^\star_p - \sigma^{\star, {\rm rev}}_p)}
        \right)\\
        & \Leftrightarrow \exists \sigma \in S_\delta, \forall \nu \in N_\delta, \mkendall{(\sigma - \nu)}{p} 
        \leq 
        \varepsilon\left(
             \mkendall{(\sigma - \nu)}{\sigma^\star_p}
        \right)&\!\!\!\!\!\!\!\!\!\!\!\!\!\!\!\!\!\!\!\!\!\!\!\!\!\!\!\!\!\!\!\!\!\!\!\!\!\!\!\!\!\!\!\!\!\!\!\!\!\!\!\!\!\!\!\!\!\!\!\!\!\!\!\!\!\!\!\!\!\!\!\!\!\!\!\!\!\!\!\!\text{as }\mkendall{\cdot}{\sigma^{\star, {\rm rev}}_p} = \|D_\tau\|_\infty - \mkendall{\cdot}{\sigma^{\star}_p}\\
        & \Leftrightarrow \exists \sigma \in S_\delta, \forall \nu \in N_\delta, \frac{\mkendall{(\sigma - \nu)}{p}}{\mkendall{(\sigma - \nu)}{\sigma^\star_p}}
        \leq 
        \varepsilon\\
        & \Leftrightarrow \min_{\sigma \in S_\delta}\max_{\nu \in N_\delta} \frac{\mkendall{(\sigma - \nu)}{p}}{\mkendall{(\sigma - \nu)}{\sigma^\star_p}}
        \leq 
        \varepsilon \label{eq:in_tilde_E}
    \end{align}
    Now, denoting $\varepsilon^+(\delta) = \min_{\sigma \in S_\delta}\max_{\nu \in N_\delta} \frac{\mkendall{(\sigma - \nu)}{p}}{\mkendall{(\sigma - \nu)}{\sigma^\star_p}}$, by definition $\varepsilon^+(\delta)$ satisfies \cref{eq:in_tilde_E}, which means $\varepsilon^+(\delta) \in \tilde{E}$ iff $\varepsilon^+(\delta) \leq 2 p(\sigma^\star_p)$.
    Thus, if $\varepsilon^+(\delta) \leq 2 p(\sigma^\star_p)$, then 
    \begin{align}
         \varepsilon^+(\delta) = \inf \tilde{E} \geq \inf E = \varepsilon^\star_{d_\tau, p, \sigma^\star_p}.
    \end{align}
\end{proof}

\subsection{Lower-bound}
\label{app:breakdown_function_median_lb}
We first remind \cref{thm:ubbreakdownfunctionmedian}.
\thmubbreakdownfunctionmedian*

We re-state the theorem with the matrix notation defined in \cref{app:notation}.

\begin{theorem}
For $p\in\cM_+^1(\pS)$, $d$ and $m$ two metrics on $\pS$ and $\sigma^\star_p = \sigma^{\rm med}_{d}(p)$, we have
\begin{align}
    \varepsilon^\star_{m, p, \sigma^\star_p} \geq \min_{\sigma\in S_\delta}\max_{\nu\in\pS: \nu\neq\sigma}\frac{\mdist{(\sigma-\nu)}{p}}{\|D(\sigma-\nu)\|_\infty}\,,
\end{align}
where $S_\delta = \{\sigma\in\pS | d_\tau(\sigma, \sigma^\star_p) \geq \delta\}$.

\end{theorem}
\begin{proof}
Let $S_\delta, N_\delta, E, \tilde{E}$ are defined as above.
    \begin{align}
        \varepsilon^\star_{m, p, \sigma^\star_p} & = \inf \left\{\varepsilon>0\middle|\sup_{q: \textsc{tv}(p, q)\leq \varepsilon}m(\sigma^\star_p, \sigma^\star_q) \geq \delta\right\} \\
        & = \inf \left\{\varepsilon>0\middle|\exists q, s.t. \textsc{tv}(p, q)\leq \varepsilon ~\text{and}~ m(\sigma^\star_p, \sigma^\star_q) \geq \delta\right\}\\
        & = \inf \underbrace{\left\{\varepsilon>0\middle|\exists q, s.t. \textsc{tv}(p, q)\leq \varepsilon ~\text{and}~ \argmin_{\sigma\in\pS}\mdist{\sigma}{q} \subseteq S_\delta\right\}}_{=: E} ~~~\text{with}~S_\delta = \{\sigma\in\pS | m(\sigma, \sigma^\star_p) \geq \delta\}.
    \end{align}
    Now,
    \begin{align}
        \varepsilon \in E & \Leftrightarrow \exists q, s.t. \textsc{tv}(p, q)\leq \varepsilon ~\text{and}~ \argmin_{\sigma\in\pS}\mdist{\sigma}{q} \subseteq S_\delta\\
        & \Leftrightarrow \exists q \in \Delta^\pS, \textsc{tv}(p, q)\leq \varepsilon ~\text{and}~ \exists \sigma\in S_\delta, \forall \nu\in\pS, \mdist{\sigma}{q} \leq \mdist{\nu}{q}\\
        & \Leftrightarrow \exists q \in \Delta^\pS, \textsc{tv}(p, q)\leq \varepsilon ~\text{and}~ \exists \sigma\in S_\delta, \forall \nu\in\pS, \mdist{(\sigma-\nu)}{p} \leq \mdist{(\sigma-\nu)}{(q_- - q_+)}\\
        & ~~~~~~~\text{where}~ q_+ = (q-p)_+ ~~\text{and}~~ q_- = (p-q)_+\nonumber\\
        & \Rightarrow \exists q \in \Delta^\pS, \textsc{tv}(p, q)\leq \varepsilon ~\text{and}~ \exists \sigma\in S_\delta, \forall \nu\in\pS, \mdist{(\sigma-\nu)}{p} \leq \|q_+-q_-\|_1 \|D(\sigma-\nu)\|_\infty\\
        & \Rightarrow \exists \sigma\in S_\delta, \forall \nu\in\pS, \mdist{(\sigma-\nu)}{p} \leq \varepsilon\|D(\sigma-\nu)\|_\infty & \!\!\!\!\!\!\!\!\!\!\!\!\!\!\!\!\!\!\!\!\!\!\!\!\!\!\!\!\!\!\!\!\!\text{as }\|q_+-q_-\|_1\leq \varepsilon\\
        & \Rightarrow \exists \sigma\in S_\delta, \forall \nu\in\pS, s.t. \sigma \neq \nu, \frac{\mdist{(\sigma-\nu)}{p}}{\|D(\sigma-\nu)\|_\infty} \leq \varepsilon\\
        & \Rightarrow \min_{\sigma\in S_\delta}\max_{\nu\in\pS: \nu\neq\sigma}\frac{\mdist{(\sigma-\nu)}{p}}{\|D(\sigma-\nu)\|_\infty} \leq \varepsilon.
    \end{align}
    Finally,
    \begin{align}
        \varepsilon^\star_{m, p, \sigma^\star_p} & = \inf E \geq \min_{\sigma\in S_\delta}\max_{\nu\in\pS: \nu\neq\sigma}\frac{\mdist{(\sigma-\nu)}{p}}{\|D(\sigma-\nu)\|_\infty}\,.
    \end{align}
\end{proof}

\section{Hausdorff Extensions of Kendall Tau}
\label{app:hausdorff_kendall}

We remind first the Kendall-tau distance, defined by: $$d_{\tau}: (\sigma_1, \sigma_2) \in \frak{S}_n \times \frak{S}_n \to \sum_{i<j} \mathbb{1}( (\sigma_1(i)-\sigma_1(j))(\sigma_2(i)-\sigma_2(j)) < 0 )$$  and the  \cref{def:non_symmetric_hausdorff,def:symmetric_hausdorff} of the Hausdorff extensions of the Kendall tau metric.

\defnonsymmetrichausdorff*

\defsymmetrichausdorff*

\begin{restatable}{proposition}{propcomplexityhausdorffkendall}\label{prop:complexity_hausdorff_kendall}
For any $\pi_1, \pi_2\in\wO$, the computation cost of $H_{d_\tau}^{\textsc{ns}}(\pi_1, \pi_2)$ and $H_{d_\tau}^{(1/2)}(\pi_1, \pi_2)$ is $\cO(n^2)$.
\end{restatable}

The average Hausdorff distance can be expressed with various expressions, necessitating the following notations (see \cite{fagin2006comparing}): 
\begin{enumerate}
    \item $\forall \, i \in [\![1,n]\!] \quad \bar{\pi}(i) = \sum_{\sigma \in \pi} \sigma(i)$ is the rank of item $i$ according to weak order $\pi$.
    \item $S(\pi_1, \pi_2) = \{ (i <j ) \; | \; \bar{\pi}_1(i)\neq\bar{\pi}_1(j), [\bar{\pi}_1(i)-\bar{\pi}_1(j)][\bar{\pi}_2(i)-\bar{\pi}_2(j)] < 0 \}$ is the set of item pairs $(i<j)$ that are in different buckets in both $\pi_1$ and $\pi_2$, and that are in different orders in $\pi_1$ and $\pi_2$.
    \item $S(\pi_1 \setminus \pi_2) = \{(i<j) \; | \; \bar{\pi}_1(i) = \bar{\pi}_1(j) \text{ and } \bar{\pi}_2(i) \neq \bar{\pi}_2(j) \}$ is the set of item pairs $(i<j)$ such that both items are in the same bucket in $\pi_1$ but in different ones in $\pi_2$.
    \item $\prof(\pi) = (\prof(\pi)_{i,j})_{i<j}$, where $\forall \; i<j, \prof(\pi)_{i,j} = 1/2$ if $\bar{\pi}(i) < \bar{\pi}(j)$, $= 0$ if $\bar{\pi}(i) = \bar{\pi}(j)$ and $= -1/2$ if $\bar{\pi}(i) > \bar{\pi}(j)$. $\prof(\pi)$ is called the profile vector of $\pi$.
\end{enumerate}

We have the following equivalent expressions for the average Hausdorff distance:

\begin{restatable}[Average Hausdorff distance]{proposition}{propavghausdorff}\label{prop:avg_hausdorff_expressions}
\begin{align}
    H_K^{(1/2)}(\pi_1, \pi_2) &:= \# S(\pi_1, \pi_2) + \frac{1}{2} \left( \#S(\pi_1 \setminus \pi_2) + \#S(\pi_2 \setminus \pi_1) \right) \\
    &= \sum_{i<j} \mathbb{1}\left( [\bar{\pi}_1(i)-\bar{\pi}_1(j)][\bar{\pi}_2(i)-\bar{\pi}_2(j)] < 0 \right) + \nonumber \\
    & \quad \quad \quad \frac{1}{2} \mathbb{1}\left( [\bar{\pi}_1(i)=\bar{\pi}_1(j)] \right)\mathbb{1}\left( [\bar{\pi}_2(i)\neq\bar{\pi}_2(j)] \right) + \nonumber \\
    & \quad \quad \quad\frac{1}{2} \mathbb{1}\left( [\bar{\pi}_2(i)=\bar{\pi}_2(j)] \right)\mathbb{1}\left( [\bar{\pi}_1(i)\neq\bar{\pi}_1(j)] \right) \\
    &= \| \prof(\pi_1) - \prof(\pi_2) \|_1
    \label{eq:avg_hausdorff_expressions}
\end{align}
\end{restatable}

\begin{restatable}[Avergage Hausdorff distance - Proof]{proof}{proofavghausdorff}\label{proof:avg_hausdorff_expressions}
Let $\pi_1$, $\pi_2$ be two weak orders associated with buckets $(B^1_1,...B^1_{t_1})$ and $(B^2_1,...B^2_{t_2})$ respectively. Such buckets are sets of items $i$ forming a partition of $[\!1,n]\!]$ such that $i \in B^1_k$ iif $\bar{\pi}_1(i) = \sum_{k'< k} \#B^1_{k'} + \frac{\#B^1_k + 1}{2}$  (see \cite{fagin2006comparing} for a more formal definition). Let us define, as in \cite{critchlow2012metric, fagin2006comparing}, ${\forall \; i \leq t_1, \forall \; j \leq t_2, \quad n_{i,j}= \#(B_i \cap B_j)}$.

Then we have \cite{critchlow2012metric}[Chapter IV]: $H_K^{(1/2)} = \frac{1}{2} \left( \sum_{i<i', j \geq j'}n_{i,j}n_{i',j'} + \sum_{i \leq i', j > j'}n_{i,j}n_{i',j'} \right)$.

By noting that $2 \# S(\pi_1, \pi_2) = \sum_{i<i', j > j'}n_{i,j}n_{i',j'}$ and $2 \#S(\pi_1 \setminus \pi_2) = \sum_{i=i', j > j'}n_{i,j}n_{i',j'}$, we derive our first equality. The second equality directly comes from re-expressing the first one. The third equality comes from \cite{fagin2006comparing}.

\end{restatable}



\end{document}